\documentclass{article}

\usepackage[final]{nips_2017}

\usepackage[utf8]{inputenc} 
\usepackage[T1]{fontenc}    
\usepackage{hyperref}       
\usepackage{url}            
\usepackage{booktabs}       
\usepackage{amsfonts}       
\usepackage{nicefrac}       
\usepackage{microtype}      

\usepackage{algorithm, algorithmic}
\usepackage{xspace}
\usepackage{enumerate, paralist}
\usepackage{comment}
\usepackage{amsmath, amssymb, amsthm, bm}
\newenvironment{talign*}
{\let\displaystyle\textstyle\csname align*\endcsname}
{\endalign}

\makeatletter
\def\blfootnote{\xdef\@thefnmark{*}\@footnotetext}
\makeatother

\title{Repeated Inverse Reinforcement Learning}

\author{
  Kareem Amin\thanks{Equal contribution.} \\
  Google Research \\
  New York, NY 10011\\
  \texttt{kamin@google.com} \\
  \And
  Nan Jiang$^*$ \hspace*{2em} Satinder Singh \\
  Computer Science \& Engineering, \\
  University of Michigan, 
  Ann Arbor, MI 48104\\
  \texttt{\{nanjiang,baveja\}@umich.edu} \\
}

\newcommand{\Demoer}{Human\xspace}
\newcommand{\demoer}{human\xspace}

\newcommand{\Learner}{Agent\xspace}
\newcommand{\learner}{agent\xspace}

\newcommand{\trueR}{\theta_\star}
\newcommand{\RR}{\mathbb{R}}
\newcommand{\EE}{\mathbb{E}}
\newcommand{\Scal}{\mathcal{S}}
\newcommand{\Acal}{\mathcal{A}}
\newcommand{\Abdt}{\mathcal{D}}
\newcommand{\occ}{\eta}
\newcommand{\bone}{\mathbf{1}}
\newcommand{\bzro}{\mathbf{0}}
\newcommand{\sref}{s_{\textrm{ref}}}
\newcommand{\ii}[1]{^{(#1)}}
\newcommand{\ee}{\bm{e}}
\newcommand{\del}[1]{^{\setminus #1}}

\newcommand{\mvee}{\textrm{MVEE}}
\newcommand{\vol}{\textrm{vol}}
\newcommand{\spread}{\textrm{spread}}

\theoremstyle{plain}
\newtheorem{theorem}{\textbf{Theorem}}
\newtheorem{lemma}{\textbf{Lemma}}

\newtheorem{proposition}{Proposition}

\theoremstyle{definition}
\newtheorem{definition}{Definition}

\newtheorem{example}{Example}

\begin{document}
\blfootnote{This paper extends an unpublished arXiv paper by the authors \citep{amin2016towards}.}

\maketitle

\begin{abstract}
We introduce a novel \emph{repeated} Inverse Reinforcement Learning problem: the agent has to act on behalf of a human in a sequence of tasks and wishes to minimize the number of tasks that it surprises the human by acting suboptimally with respect to how the human would have acted. Each time the human is surprised, the agent is provided a demonstration of the desired behavior by the human. We formalize this problem, including how the sequence of tasks is chosen, in a few different ways and provide some foundational results.  
\end{abstract}

\section{Introduction}

One challenge in building AI agents that learn from experience is how to set their goals or rewards. In the Reinforcement Learning (RL) setting, one interesting answer to this question is inverse RL (or IRL) in which the agent infers the rewards of a human by observing the human's policy in a task \citep{ng2000algorithms}. Unfortunately, the IRL problem is ill-posed for there are typically many reward functions for which the observed behavior is optimal in a single task \citep{abbeel2004apprenticeship}. While the use of heuristics to select from among the set of feasible reward functions has led to successful applications of IRL to the problem of learning from demonstration \citep[e.g.,][]{abbeel2007application}, not identifying the reward function poses fundamental challenges to the question of how well and how safely the agent will perform when using the learned reward function in other tasks.

We formalize multiple variations of a \emph{new repeated IRL problem} in which the agent and (the same) human face multiple tasks over time. We separate the reward function into two components, one which is invariant across tasks and can be viewed as intrinsic to the human, and a second that is task specific. As a motivating example, consider a human doing tasks throughout a work day, e.g., getting coffee, driving to work, interacting with co-workers, and so on. Each of these tasks has a task-specific goal, but the human brings to each task intrinsic goals that correspond to maintaining health, financial well-being, not violating moral and legal principles, etc. In our repeated IRL setting, the agent presents a policy for each new task that it thinks the human would do. If the agent's policy ``surprises'' the human by being sub-optimal, the human presents the agent with the optimal policy. The objective of the agent is to minimize the number of surprises to the human, i.e., to generalize the human's behavior to new tasks. 

In addition to addressing generalization across tasks, the repeated IRL problem we introduce and our results are of interest in resolving the question of unidentifiability of rewards from observations in standard IRL. Our results are also of interest to a particular aspect of the concern about how to make sure that the AI systems we build are safe, or AI safety. Specifically, the issue of reward misspecification is often mentioned in AI safety articles \citep[e.g.,][]{bostrom2003ethical, russell2015research, amodei2016concrete}. These articles mostly discuss broad ethical concerns and possible research directions, while our paper develops mathematical formulations and algorithmic solutions to a specific way of addressing reward misspecification. 

In summary form, our contributions include: (1) an efficient reward-identification algorithm when the agent can choose the tasks in which it observes human behavior; (2) an upper bound on the number of total surprises when no assumptions are made on the tasks, along with a corresponding lower bound; (3) an extension to the setting where the human provides sample trajectories instead of complete behavior; and (4) identification guarantees when the agent can only choose the task rewards but is given a fixed task environment.

\section{Markov Decision Processes (MDPs)}
An MDP is specified by its state space $\Scal$, action space $\Acal$, initial state distribution $\mu \in \Delta(\Scal)$, transition function (or dynamics) $P: \Scal \times \Acal \to \Delta(\Scal)$, reward function $Y: \Scal \to \RR$,  
and discount factor $\gamma \in [0, 1)$. We assume finite $\Scal$ and $\Acal$, and $\Delta(\Scal)$ is the space of all distributions over $\Scal$. A policy $\pi: \Scal\to\Acal$ describes an agent's behavior by specifying the action to take in each state. The (normalized) value function or long-term utility of $\pi$ is defined as $V^\pi(s) = (1-\gamma) \, \EE [\sum_{t=1}^\infty \gamma^{t-1} Y(s_t) | s_0 = s; \pi]$.\footnote{Here we differ (w.l.o.g.) from common IRL literature in assuming that reward occurs after transition. 
} 
Similarly, the Q-value function is $Q^\pi(s, a) = (1-\gamma) \,\EE [\sum_{t=1}^\infty \gamma^{t-1} Y(s_t) | s_0 = s, a_0 = a; \pi]$. Where necessary we will use the notation $V_{P, Y}^\pi$ to avoid ambiguity about the dynamics and the reward function. 
Let $\pi^\star: \Scal \to \Acal$ be an optimal policy, which maximizes $V^\pi$ and $Q^\pi$ in all states (and actions) simultaneously.

Given an initial distribution over states, $\mu$, a scalar value that measures the goodness of $\pi$ is defined as $\EE_{s\sim \mu} [V^\pi(s)]$. We introduce some further notation to express $\EE_{s\sim \mu} [V^\pi(s)]$ in vector-matrix form. Let $\occ_{\mu, P}^\pi \in \RR^{|\Scal|}$ be the \emph{normalized state occupancy} under initial distribution $\mu$, dynamics $P$, and policy $\pi$, whose $s$-th entry is $(1-\gamma)\, \EE [\sum_{t=1}^\infty \gamma^{t-1} \mathbb{I}(s_t = s) | s_0 \sim \mu; \pi]$ ($\mathbb{I}(\cdot)$ is the indicator function). This vector can be computed in closed-form as 
$
\occ_{\mu, P}^\pi = (1-\gamma) \left(\mu^\top P^\pi \, \left(\mathbf{I}_{|\Scal|} - \gamma P^\pi\right)^{-1} \right)^\top,
$ 
where $P^\pi$ is an $|\Scal|\times|\Scal|$ matrix whose $(s, s')$-th element is $P(s' | s, \pi(s))$, and $\mathbf{I}_{|\Scal|}$ is the $|\Scal|\times |\Scal|$ identity matrix. 
For convenience we will also treat the reward function $Y$ as a vector in $\RR^{|\Scal|}$, and we have
\begin{align} \label{eq:value}
\EE_{s\sim \mu} [V^\pi(s)] = Y^\top \occ_{\mu, P}^{\pi}.
\end{align}

\section{Problem setup}
Here we define the \textbf{repeated IRL problem}. 
The human's reward function $\trueR$ captures his/her safety concerns and intrinsic/general preferences. This $\trueR$ is unknown to the agent and is the object of interest herein, i.e., if $\trueR$ were known to the agent, the concerns addressed in this paper would be solved. 
We assume that the human cannot directly communicate $\trueR$ to the agent but can evaluate the agent's behavior in a task as well as demonstrate optimal behavior. 
Each task comes with an external reward function $R$, and the goal is to maximize the reward with respect to $Y := \trueR+R$ in each task. 

As a concrete example, consider an agent for an autonomous vehicle. In this case, $\trueR$ represents the cross-task principles that define good driving (e.g., courtesy towards pedestrians and other vehicles), which are often difficult to explicitly describe. In contrast, $R$, the task-specific reward, could reward the agent for successfully completing parallel parking. While $R$ is easier to construct, it may not completely capture what a human deems good driving. (For example, an agent might successfully parallel park while still boxing in neighboring vehicles.)

More formally, a task is defined by a pair $(E, R)$, where $E = (\Scal,\Acal, \mu, P,\gamma)$ is the task environment (i.e., 
a controlled Markov process) 
and $R$ is the task-specific reward function (\emph{task reward}). We assume that all tasks share the same $\Scal, \Acal, \gamma$, with $|\Acal|\ge 2$, 
but \textbf{may} differ in the initial distribution $\mu$, dynamics $P$, and task reward $R$; all of the task-specifying quantities are known to the agent. In any task, the human's optimal behavior is always with respect to the reward function $Y = \trueR + R$. We emphasize again that $\trueR$ is intrinsic to the human and remains the same across all tasks. Our use of task specific reward functions $R$ allows for \textbf{greater generality} than the usual IRL setting, and most of our results apply equally to the case where $R \equiv \bzro$.

While $\trueR$ is private to the human, the agent has some prior knowledge on $\trueR$, represented as a set of possible parameters $\Theta_0 \subset \RR^{|\Scal|}$ that contains $\trueR$. Throughout, we assume that the human's reward has bounded and normalized magnitude, that is, $\|\trueR\|_\infty \le 1$.

A demonstration in $(E, R)$ reveals $\pi^\star$, optimal for $Y = \trueR + R$ under environment $E$, to the \learner.
A common assumption in the IRL literature is that the full mapping is revealed, which can be unrealistic if some states are unreachable from the initial distribution. 
We address the issue 
by requiring only the state occupancy vector $\occ_{\mu, P}^{\pi^*}$. In 
Section~\ref{sec:extension} we show that this also allows an easy extension to the setting where the human only demonstrates  trajectories instead of providing a policy.

Under the above framework for repeated IRL, we consider two settings that differ in how the sequence of tasks are chosen. 
In both settings, we will want to minimize the number of demonstrations needed. 

{\bf 1.} (Section~\ref{sec:omni}) \emph{Agent chooses the tasks}, observes the human's behavior in each of them, and infers the reward function. In this setting where the agent is powerful enough to choose tasks arbitrarily, we will show that the agent will be able to {\em identify\/} the human's reward function which of course implies the ability to generalize to new tasks. 

{\bf 2.} (Section~\ref{sec:passive}) \emph{Nature chooses the tasks}, and the agent proposes a policy in each task. The human demonstrates a policy only if the agent's policy is significantly suboptimal (i.e., {\bf a mistake}). In this setting we will derive upper and lower bounds on the number of mistakes our agent will make.

\section{The challenge of identifying rewards}

Note that it is impossible to identify $\trueR$ from watching human behavior in a single task. This is because any $\trueR$ is fundamentally indistinguishable from an infinite set of reward functions that yield exactly the policy observed in the task. We introduce the idea of \emph{behavioral equivalence} below to tease apart two separate issues wrapped up in the challenge of identifying rewards.

\begin{definition} \label{def:equiv}
	Two reward functions $\theta, \theta'\in \RR^{|\Scal|}$ are \emph{behaviorally equivalent in all MDP tasks}, if for any $(E, R)$, the set of optimal policies for $(R+\theta)$ and $(R+\theta')$ are the same.
\end{definition}

We argue that the task of identifying the reward function should amount only to identifying the (behavioral) equivalence class to which $\trueR$ belongs. In particular, identifying the equivalence class is sufficient to get perfect generalization to new tasks. Any remaining unidentifiability is merely representational and of no real consequence. Next we present a constraint that captures the reward functions that belong to the same equivalence class. 


\begin{proposition} \label{prop:equiv}
	Two reward functions $\theta$ and $\theta'$ are behaviorally equivalent in all MDP tasks if and only if $\theta - \theta' = c\cdot \bone_{|\Scal|}$ for some $c\in \RR$, where $\bone_{|\Scal|}$ is an all-1 vector of length $|\Scal|$.
\end{proposition}
The proof is elementary and deferred to Appendix~\ref{app:equivalence}. 
For any class of $\theta$'s that are equivalent to each other, we can choose a canonical element to represent this class. For example, we can fix an arbitrary reference state $\sref \in \Scal$, and fix the reward of this state to $0$ for $\trueR$ and all candidate $\theta$'s. In the rest of the paper, we will always assume such canonicalization in the MDP setting, hence $\trueR \in \Theta_0 \subseteq \{\theta \in [-1, 1]^{|\Scal|}: \theta(\sref) = 0 \}$. 

\section{Agent chooses the tasks}
\label{sec:omni}
In this section, the protocol is that the \learner chooses a sequence of tasks $\{(E_t, R_t)\}$. For each task $(E_t, R_t)$, the \demoer reveals $\pi_t^\star$, which is optimal for environment $E_t$ and reward function $\trueR + R_t$. Our goal is to design an algorithm which chooses $\{(E_t, R_t)\}$ and identifies $\trueR$ to a desired accuracy, $\epsilon$, using as few tasks as possible. 
%
%
Theorem~\ref{thm:omni} shows that a simple algorithm can identify $\trueR$ after only $O(\log(1/\epsilon))$ tasks, if \emph{any} tasks may be chosen. Roughly speaking, the algorithm amounts to a binary search on each component of $\trueR$ by manipulating the task reward $R_t$.\footnote{While we present a proof that manipulates $R_t$, an only slightly more complex proof applies to the setting where all the $R_t$ are exactly zero and the manipulation is limited to the environment \citep{amin2016towards}.} 
See the proof for the algorithm specification. 
As noted before, once the agent has identified $\trueR$ within an appropriate tolerance, it can compute a sufficiently-near-optimal policy for all tasks, thus completing the generalization objective through the far stronger identification objective in this setting. 

\begin{theorem} \label{thm:omni}
	If $\trueR \in \Theta_0 \subseteq \{\theta \in [-1, 1]^{|\Scal|}: \theta(\sref) = 0 \}$, there exists an algorithm that outputs $\theta \in \RR^{|\Scal|}$ that satisfies $\|\theta - \trueR\|_\infty \le \epsilon$ after $O(\log(1/\epsilon))$ demonstrations.
\end{theorem}
\vspace*{-1.2em}
\begin{proof}
The algorithm chooses the following fixed environment in all tasks: for each $s\in \Scal \setminus \{\sref\}$, let one action be a self-loop, and the other action transitions to $\sref$. In $\sref$, all actions cause self-loops. The initial distribution over states is uniformly at random over $\Scal \setminus \{\sref\}$. 
Each task only differs in the task reward $R_t$ (where $R_t(\sref) \equiv 0$ always). After observing the state occupancy of the optimal policy, for each $s$ we check if the occupancy is equal to $0$. If so, it means that the demonstrated optimal policy chooses to go to $\sref$ from $s$ in the first time step, and $\trueR(s) + R_t(s) \le \trueR(\sref) + R_t(\sref) = 0$; if not, we have $\trueR(s) + R_t(s) \ge 0$. Consequently, after each task we learn the relationship between $\trueR(s)$ and $-R_t(s)$ on each $s\in \Scal \setminus \{\sref\}$, so conducting a binary search by manipulating $R_t(s)$ will identify $\trueR$ to $\epsilon$-accuracy after $O(\log(1/\epsilon))$ tasks.
\end{proof}

\section{Nature chooses the tasks}
\label{sec:passive}
While Theorem~\ref{thm:omni} yields a strong identification guarantee, it also relies on a strong assumption, that $\{(E_t, R_t)\}$ may be chosen by the \learner in an arbitrary manner. In this section, we let \emph{nature}, who is allowed to be adversarial for the purpose of the analysis, choose $\{(E_t, R_t)\}$.

Generally speaking, we cannot obtain identification guarantees in such an adversarial setup. As an example, if $R_t \equiv 0$ and $E_t$ remains the same over time, we are essentially back to the classical IRL setting and suffer from the degeneracy issue. However, generalization to future tasks, which is our ultimate goal, is easy in this special case: after the initial demonstration, the \learner can mimic it to behave optimally in all subsequent tasks without requiring further demonstrations. 
More generally, if nature repeats similar tasks, then the \learner obtains little new information, but presumably it knows how to behave in most cases; if nature chooses a task unfamiliar to the \learner, then the \learner is likely to err, but it may learn about $\trueR$ from the mistake.

To formalize this intuition, we consider the following protocol: the nature chooses a sequence of tasks $\{(E_t, R_t)\}$ in an arbitrary manner. For every task $(E_t, R_t)$, the \learner proposes a policy $\pi_t$. The \demoer examines the policy's value under $\mu_t$, and if the loss
\begin{align} \label{eq:mdploss}
l_t = \EE_{s\sim \mu} \left[V_{E_t, \, \trueR + R_t}^{\pi_t^\star}(s)\right] -  \EE_{s\sim \mu} \left[V_{E_t, \, \trueR + R_t}^{\pi_t}(s) \right] 
\end{align}
is less than some $\epsilon$ then the \demoer is satisfied and no demonstration is needed; otherwise a mistake is counted and $\occ_{\mu_t, P_t}^{\pi_t^\star}$ is revealed to the \learner (note that $\occ_{\mu_t, P_t}^{\pi_t^\star}$ can be computed by the agent if needed from $\pi_t^*$ and its knowledge of the task). 
The main goal of this section is to design an algorithm that has a provable guarantee on the total number of mistakes.

\textbf{On human supervision}~~Here we require the \demoer to evaluate the agent's policies in addition to providing demonstrations. 
We argue that this is a reasonable assumption because (1) only a binary signal $\mathbb{I}(l_t > \epsilon)$ is needed as opposed to the precise value of $l_t$, and (2) if a policy is suboptimal but the \demoer fails to realize it, arguably it should not be treated as a mistake. 
Meanwhile, we will also provide identification guarantees in Section~\ref{sec:ellip_iden}, as the \demoer will be relieved from the supervision duty once $\trueR$ is identified.

Before describing and analyzing our algorithm, we first notice that the Equation~\ref{eq:mdploss} can be rewritten as 
\begin{align}
l_t = (\trueR + R)^\top (\occ_{\mu_t, P_t}^{\pi_t^\star} - \occ_{\mu_t, P_t}^{\pi_t}),
\end{align}
using Equation~\ref{eq:value}. 
So effectively, the given environment $E_t$ in each round induces a set of state occupancy vectors $\{\occ_{\mu_t, P_t}^\pi: \pi \in (\Scal \to \Acal)\}$, and we want the agent to choose the vector that has the largest dot product with $\trueR + R$. 
The exponential size of the set will not be a concern because our main result (Theorem~\ref{thm:ellipsoid}) has no dependence on the number of vectors, and only depends on the dimension of those vectors. 
The result is enabled by studying the \emph{linear bandit} version of the problem, which subsumes the MDP setting for our purpose and is also a model of independent interest. 


\subsection{The linear bandit setting} \label{sec:lin_bandit}
In the linear bandit setting,  $\Abdt$ is a finite action space with size $|\Abdt| = K$. Each task is denoted as a pair $(X, R)$, where $R$ is the task specific reward function as before. $X = [x\ii{1}~ \cdots~ x\ii{K}]$ is a $d\times K$ feature matrix, where $x\ii{i}$ is the feature vector for the $i$-th action, and $\|x\ii{i}\|_1 \le 1$. When we reduce MDPs to linear bandits, each element of $\Abdt$ corresponds to an MDP policy, and the feature vector is the state occupancy of that policy. 

As before, $R, \trueR \in \RR^d$ are the task reward and the human's unknown reward, respectively. The initial uncertainty set for $\trueR$ is $\Theta_0 \subseteq [-1,1]^d$. The value of the $i$-th action is calculated as $(\trueR + R)^\top x\ii{i}$, and $a^\star$ is the action that maximizes this value. Every round the \learner proposes an action $a\in \Abdt$, whose loss is defined as
$$
l_t = (\trueR + R)^\top (x^{a^\star} - x^{a}).
$$
We now show how to embed the previous MDP setting in linear bandits. 
\begin{example} \label{exm:mdp2bandit}
Given an MDP problem with variables $\Scal, \Acal, \gamma, \trueR,  \sref, \Theta_0, \{(E_t, R_t)\}$, we can convert it into a linear bandit problem as follows: (all variables with prime belong to the linear bandit problem, and we use $v\del{i}$ to denote the vector $v$ with the $i$-th coordinate removed)  \vspace*{-.25em}
\begin{compactitem}
\item $\Abdt = \{\pi: \Scal \to \Acal\}$, $d = |\Scal| - 1$, 
$\trueR' = \trueR\del{\sref}, \Theta_0' = \{\theta\del{\sref}: \theta \in \Theta_0\}$.
\item $x_t^\pi = (\occ_{\mu_t, P_t}^\pi)\del{\sref}$. $R_t' = R_t\del{\sref} - R_t(\sref) \cdot \bone_{d}$.
\end{compactitem}
\vspace*{-.5em}
\end{example} 
Note that there is a more straightforward conversion by letting $d=|\Scal|, \trueR' = \trueR, \Theta_0' = \Theta_0, x_t^\pi = \occ_{\mu_t, P_t}^\pi, R_t' = R_t$, which also preserves losses. We perform a more succinct conversion in Example~\ref{exm:mdp2bandit} by canonicalizing both $\trueR$ (already assumed) and $R_t$ (explicitly done here) and dropping the coordinate for $\sref$ in all relevant vectors. 



\paragraph{MDPs with linear rewards} In IRL literature, a generalization of the MDP setting is often considered, that reward is linear in state features $\phi(s) \in \RR^d$  \citep{ng2000algorithms, abbeel2004apprenticeship}. In this new setting, $\trueR$ and $R$ are reward parameters, and the actual reward is $(\trueR + R)^\top \phi(s)$. 
This new setting can also be reduced to linear bandits similarly to Example~\ref{exm:mdp2bandit}, except that the state occupancy is replaced by the discounted sum of expected feature values. Our main result, Theorem~\ref{thm:ellipsoid}, will still apply automatically, but now the guarantee will only depend on the dimension of the feature space and has no dependence on $|\Scal|$. We include the conversion below  but do not further discuss this setting in the rest of the paper.
\begin{example} \label{exm:featmdp2bandit}
	Consider an MDP problem with state features, defined by $\Scal, \Acal, \gamma, d \in \mathbb{Z}^+, \trueR \in \RR^d, \Theta_0 \subseteq [-1, 1]^d, \{(E_t, \phi_t \in \RR^d, R_t \in \RR^d)\}$, where task reward and background reward in state $s$ are $\trueR^\top \phi_t(s)$ and $R^\top \phi_t(s)$ respectively, and $\trueR \in \Theta_0$. Suppose $\|\phi_t(s)\|_\infty \le 1$ always holds, then we can convert it into a linear bandit problem as follows:
$\Abdt = \{\pi: \Scal \to \Acal\}$. $d$, $\trueR$, and $R_t$ remain the same.
$x_t^\pi = (1-\gamma) \sum_{h=1}^\infty \gamma^{h-1} \EE[\phi(s_h)\,|\, \mu_t, P_t, \pi ] / d$. 
Note that the division of $d$ in $x_t^\pi$ is for the purpose of normalization, so that $\|x_t^\pi\|_1 \le \|\phi\|_1 / d \le \|\phi\|_\infty \le 1$. 
\end{example}

\subsection{Ellipsoid Algorithm for Repeated Inverse Reinforcement Learning}
\label{sec:ellipsoid}
We propose Algorithm~\ref{alg:ellipsoid}, and provide the mistake bound in the following theorem. 

\begin{algorithm}[tb]
\caption{Ellipsoid Algorithm for Repeated Inverse Reinforcement Learning}
\label{alg:ellipsoid}
\begin{algorithmic}[1]
\STATE {\bfseries Input:} $\Theta_0$.
\STATE $\Theta_1 \leftarrow \mvee(\Theta_0)$.
\FOR{$t=1,2,\ldots$}
\STATE Nature reveals $(X_t, R_t)$. 
\STATE Learner plays $a_t = \arg\max_{a\in\Abdt}\, c_t^\top x_t^{a}$, where $c_t$ is the center of $\Theta_t$. \label{lin:c_t_def}
 $\Theta_{t+1} \leftarrow \Theta_t$.
\IF{$l_t > \epsilon$}
\STATE \Demoer reveals $a_t^\star$.~~
$\Theta_{t+1} \leftarrow
\mvee(\{\theta \in \Theta_t: (\theta - c_t)^\top (x_t^{a_t^\star} - x_t^{a_t}) \ge 0 \}).$ \label{lin:update}
\ENDIF
\ENDFOR
\end{algorithmic}
\end{algorithm}

\begin{theorem} \label{thm:ellipsoid}
	For $\Theta_0 = [-1, 1]^d$, the number of mistakes made by Algorithm~\ref{alg:ellipsoid} is guaranteed to be $O(d^2 \log(d/\epsilon))$.
\end{theorem}

To prove Theorem~\ref{thm:ellipsoid}, we quote a result from linear programming literature in Lemma~\ref{lem:vol}, which is found in standard lecture notes (e.g., \citep{ellipsoidnotes}, Theorem 8.8; see also \citep{grotschel2012geometric}, Lemma 3.1.34).
\begin{lemma}[Volume reduction in ellipsoid algorithm] \label{lem:vol}
	Given any non-degenerate ellipsoid $B$ in $\RR^d$ centered at $c \in \RR^d$, and any non-zero vector $v \in \RR^d$, let $B^+$ be the minimum-volume enclosing ellipsoid (MVEE) of
	$
	\{ u \in B: (u - c)^\top v \ge 0 \}.
	$ 
	We have $\displaystyle \vol(B^+) / \vol(B) \le e^{-\frac{1}{2(d+1)}}$.
\end{lemma}

\begin{proof}[Proof of Theorem~\ref{thm:ellipsoid}]
	Whenever a mistake is made, 
	we can induce the constraint 
	$
	(R_t + \trueR)^\top (x_t^{a_t^\star} - x_t^{a_t}) > \epsilon.
	$ 
	Meanwhile, since $a_t$ is greedy w.r.t.~$c_t$, we have
	$
	(R_t + c_t)^\top (x_t^{a_t^\star} - x_t^{a_t}) \le 0,
	$ 
	where $c_t$ is the center of $\Theta_t$ as in Line~\ref{lin:c_t_def}. 
	Taking the difference of the two inequalities, we obtain 
	\begin{align} \label{eq:constraint_tight}
	(\trueR - c_t)^\top (x_t^{a_t^\star} - x_t^{a_t}) > \epsilon.
	\end{align}
	Therefore, the update rule on Line~\ref{lin:update} of Algorithm~\ref{alg:ellipsoid} preserves $\trueR$ in $\Theta_{t+1}$. Since the update makes a central cut through the ellipsoid, Lemma~\ref{lem:vol} applies and the volume shrinks 
	every time a mistake is made. 	
	To prove the theorem, it remains to upper bound the initial volume and lower bound the terminal volume of $\Theta_t$. We first show that an update never eliminates $B_{\infty}(\trueR, \epsilon / 2)$, the $\ell_\infty$ ball centered at $\trueR$ with radius $\epsilon/2$. This is because, any eliminated $\theta$ satisfies $ (\theta + c_t)^\top (x_t^{a_t^\star} - x_t^{a_t}) < 0$. 
	Combining this with Equation~\ref{eq:constraint_tight}, we have 
	\begin{align*}
	\epsilon < (\theta^{\star} - \theta)^\top (x_t^{a_t^\star} - x_t^{a_t})
	\le \|\trueR - \theta\|_\infty \|x_t^{a_t^\star} - x_t^{a_t}\|_1
	\le 2 \|\trueR - \theta\|_\infty.
	\end{align*}
	The last step follows from $\|x\|_1 \le 1$. We conclude that any eliminated $\theta$ should be $\epsilon/2$ far away from $\trueR$ in $\ell_\infty$ distance. Hence, we can lower bound the volume of $\Theta_t$ for any $t$ by that of $\Theta_0 \bigcap B_{\infty}(\trueR, \epsilon/2)$, which contains an $\ell_\infty$ ball with radius $\epsilon/4$ at its smallest (when $\trueR$ is one of $\Theta_0$'s vertices). To simplify calculation, we relax this lower bound (volume of the $\ell_\infty$ ball) to the volume of the inscribed $\ell_2$ ball. 
	
	Finally we put everything together: let $M_T$ be the number of mistakes made from round $1$ to $T$, $C_d$ be the volume of the unit hypersphere in $\RR^d$ (i.e., $\ell_2$ ball with radius $1$), and $\vol(\cdot)$ denote the volume of an ellipsoid, we have
	\begin{align*}
    \frac{M_T}{2(d+1)} \le \log(\vol(\Theta_1)) - \log(\vol(\Theta_{T+1}))
	\le \log(C_d (\sqrt{d})^d) - \log(C_d (\epsilon/4)^d)
	= d \log \frac{4\sqrt{d}}{\epsilon}.
	\end{align*}
	So $M_T \le 2d(d+1)\log \frac{4\sqrt{d}}{\epsilon} = O(d^2 \log\frac{d}{\epsilon})$.
\end{proof}

\subsection{Lower bound} \label{sec:passive_lower_bound}
In Section~\ref{sec:omni}, we get an $O(\log(1/\epsilon))$ upper bound on the number of demonstrations, which has no dependence on $|\Scal|$ (which corresponds to $d+1$ in linear bandits). Comparing Theorem~\ref{thm:ellipsoid} to \ref{thm:omni}, one may wonder whether the polynomial dependence on $d$ is an artifact of the inefficiency of Algorithm~\ref{alg:ellipsoid}. We clarify this issue by proving a lower bound, showing that $\Omega(d \log(1/\epsilon))$ mistakes are inevitable in the worst case when nature chooses the tasks. We provide a proof sketch below, and the complete proof is deferred to Appendix~\ref{app:lower_bound}.

\begin{theorem} \label{thm:lower_bound}
	For any randomized algorithm\footnote{While our Algorithm~\ref{alg:ellipsoid} is deterministic, randomization is often crucial for online learning in general \citep{shalev2011online}.} in the linear bandit setting, there always exists $\trueR \in [-1,1]^d$ and an adversarial sequence of $\{(X_t, R_t)\}$ that potentially adapts to the algorithm's previous decisions, such that the expected number of mistakes made by the algorithm is $\Omega(d \log(1/\epsilon))$. 
\end{theorem}
\begin{proof}[Proof Sketch]
	We randomize $\trueR$  by sampling each element i.i.d.~from $\textrm{Unif}([-1,1])$. We will prove that there exists a strategy of choosing $(X_t, R_t)$ such that any algorithm's expected number of mistakes is $\Omega(d \log(1/\epsilon)$, which proves the theorem as max is no less than average.
	
	In our construction, $X_t = [\bzro_d,~ \ee_{j_t}]$, where $j_t$ is some index to be specified. Hence, every round the \learner is essentially asked to decided whether $\theta(j_t) \ge - R_t(j_t)$. The adversary's strategy goes in phases, and $R_t$ remains the same during each phase. Every phase has $d$ rounds where $j_t$ is enumerated over $\{1, \ldots, d\}$. 
	
	The adversary will use $R_t$ to shift the posterior on $\theta(j_t) + R_t(j_t)$ so that it is centered around the origin; in this way, the agent has about $1/2$ probability to make an error (regardless of the algorithm), and the posterior interval will be halved. Overall, the agent makes $d/2$ mistakes in each phase, and there will be about $\log(1/\epsilon)$ phases in total, which gives the lower bound.
\end{proof}

\textbf{Applying the lower bound to MDPs}~~ 
The above lower bound is stated for linear bandits. In principle, we need to prove lower bound for MDPs separately, because linear bandits are more general than MDPs for our purpose, and the hard instances in linear bandits may not have corresponding MDP instances. In Lemma~\ref{lem:emu} below, we show that a certain type of linear bandit instances can always be emulated by MDPs with the same number of actions, and the hard instances constructed in Theorem~\ref{thm:lower_bound} indeed satisfy the conditions for such a type; in particular, we require the feature vectors to be non-negative and have $\ell_1$ norm bounded by $1$.  As a corollary, an $\Omega(|\Scal|\log(1/\epsilon))$ lower bound for the MDP setting (even with a small action space $|\Acal|=2$) follows directly from Theorem~\ref{thm:lower_bound}. The proof of Lemma~\ref{lem:emu} is deferred to Appendix~\ref{app:emu}. 
\begin{lemma}[Linear bandit to MDP conversion] \label{lem:emu}
	Let $(X, R)$ be a linear bandit task, and $K$ be the number of actions. If every $x^{a}$ is non-negative and $\|x^{a}\|_1 \le 1$, then there exists an MDP task $(E, R')$ with $d+1$ states and $K$ actions, such that under some choice of $\sref$, converting $(E,R')$ as in Example~\ref{exm:mdp2bandit} recovers the original problem. 
\end{lemma}


\subsection{On identification when nature chooses tasks}
\label{sec:ellip_iden}
While Theorem~\ref{thm:ellipsoid} successfully controls the number of total mistakes, it completely avoids the identification problem and does not guarantee to recover $\trueR$. In this section we explore further conditions under which we can obtain identification guarantees when Nature chooses the tasks.

The first condition, stated in Proposition~\ref{prop:ellip_identify}, implies that if we have made all the possible mistakes, then we have indeed identified the $\trueR$, where the identification accuracy is determined by the tolerance parameter $\epsilon$ that defines what is counted as a mistake. Due to space limit, the proof is deferred to Appendix~\ref{app:ellip_identify}. 

\begin{proposition}\label{prop:ellip_identify}
	Consider the linear bandit setting. If there exists $T_0$ such that for any round $t \ge T_0$, no more mistakes can be ever made by the algorithm for any choice of $(E_t, R_t)$ and any tie-braking mechanism, then we have $\trueR \in B_\infty(c_{T_0}, \epsilon)$.
\end{proposition}
While the above proposition shows that identification is guaranteed if the \learner exhausts the mistakes, the \learner has no ability to actively fulfill this condition when nature chooses tasks. For a stronger identification guarantee, we may need to grant the \learner some freedom in choosing the tasks.

\textbf{Identification with fixed environment}~~
Here we consider a setting that fits in between Section~\ref{sec:omni} (completely active) and Section~\ref{sec:lin_bandit} (completely passive), where the environment $E$ (hence the induced feature vectors $\{x\ii{1}, x\ii{2}, \ldots, x\ii{K}\}$) is given and fixed, and the \learner can arbitrarily choose the task reward $R_t$. The goal is to obtain identification guarantee in this intermediate setting.

Unfortunately, a degenerate case can be easily constructed that prevents the revelation of any information about $\trueR$. In particular, if $x\ii{1} = x\ii{2} = \ldots = x\ii{K}$, i.e., the environment is completely uncontrolled, then all actions are equally optimal and nothing can be learned. 
More generally, if for some $v \ne \bzro$ we have $v^\top x\ii{1} = v^\top x\ii{2} = \ldots = v^\top x\ii{K}$, then we may never recover $\trueR$ along the direction of $v$. In fact, Proposition~\ref{prop:equiv} can be viewed as an instance of this result where $v = \bone_{|\Scal|}$ (recall that $\bone_{|\Scal|}^\top \occ_{\mu, P}^\pi \equiv 1$), and that is why we have to remove such redundancy in Example~\ref{exm:mdp2bandit} in order to discuss identification in MDPs.
Therefore, to guarantee identification in a fixed environment, the feature vectors must have significant variation in all directions, and we capture this intuition by defining a diversity score $\spread(X)$ (Definition~\ref{def:spread}) and showing that the identification accuracy depends inversely on the score (Theorem~\ref{thm:spread}).

\begin{definition}\label{def:spread}
	Given the feature matrix $X = \begin{bmatrix} x\ii{1} & x\ii{2} & \cdots & x\ii{K} \end{bmatrix}$ whose size is $d\times K$,  define $\spread(X)$ as the $d$-th largest singular value of $\widetilde X := X (\mathbf{I}_K - \frac{1}{K}\bone_K \bone_K^\top)$.
\end{definition}

\begin{theorem}\label{thm:spread}
	For a fixed feature matrix $X$, if $\spread(X)>0$, then there exists a sequence $R_1, R_2, \ldots, R_T$ with $T = O(d^2 \log(d/\epsilon))$ and a sequence of tie-break choices of the algorithm, such that after round $T$ we have 
	$ \displaystyle
	\|c_{T} - \trueR\|_\infty \le  \epsilon \sqrt{(K-1)/2} / \spread(X).
	$
\end{theorem}

The proof is deferred to Appendix~\ref{app:spread}. The $\sqrt{K}$ dependence in Theorem~\ref{thm:spread} may be of concern as $K$ can be exponentially large. However, Theorem~\ref{thm:spread} also holds if we replace $X$ by any matrix that consists of $X$'s columns, so we may choose a small yet most diverse set of columns as to optimize the bound. 

\section{Working with trajectories}
\label{sec:extension}

\begin{algorithm}[t]
	\caption{Trajectory version of Algorithm~\ref{alg:ellipsoid} for MDPs}
	\label{alg:trajectory}
	\begin{algorithmic}[1]
		\STATE {\bfseries Input:} $\Theta_0, H, n$.~~
		\STATE $\Theta_1 \leftarrow \mvee(\Theta_0)$, $i \leftarrow 0$, $\bar Z \leftarrow 0$, $\bar Z^\star \leftarrow 0$.
		\FOR{$t=1,2,\ldots$}
		\STATE Nature reveals $(E_t, R_t)$. 
		\Learner rolls-out a trajectory using $\pi_t$ greedily w.r.t.~$c_t + R_t$. 
		\STATE $\Theta_{t+1} \leftarrow \Theta_t$.
		\IF{\learner takes $a$ in $s$ with $Q^\star(s,a) < V^\star(s) - \epsilon$}
		\STATE \Demoer produces an $H$-step trajectory from $s$. Let the empirical state occupancy be $\hat z_i^{\star, H}$. 
		\STATE $i \leftarrow i + 1$, $\bar Z^\star \leftarrow \bar Z^\star + \hat z_i^{\star, H}$.
		\STATE Let $z_i$ be the state occupancy of $\pi_t$ from initial state $s$, and  $\bar Z \leftarrow \bar Z + z_i$.
		\IF{$i=n$}
		\STATE $\Theta_{t+1} \leftarrow
		\mvee(\{\theta \in \Theta_t: (\theta - c_t)^\top (\bar Z^\star - \bar Z) \ge 0 \}).$ \label{lin:traj_update} 
		~ $i \leftarrow 0$, $\bar Z \leftarrow 0$, $\bar Z^\star \leftarrow 0$.
		\ENDIF
		\ENDIF
		\ENDFOR
	\end{algorithmic}
\end{algorithm}

In previous sections, we have assumed that
the \demoer evaluates the \learner's performance based on the state occupancy of the \learner's policy, and demonstrates the optimal policy in terms of state occupancy as well. In practice, we would like to instead assume that for each task, the \learner rolls out a trajectory, and the \demoer shows an optimal trajectory if he/she finds the \learner's trajectory unsatisfying. We are still concerned about upper bounding the number of total mistakes, and aim to provide a parallel version of Theorem~\ref{thm:ellipsoid}. 

Unlike in traditional IRL, in our setting the \learner is also acting, which gives rise to many subtleties. First, the total reward on the \learner's single trajectory is a random variable, and may deviate from the expected value of its policy. Therefore, it is generally impossible to decide if the \learner's policy is near-optimal, and instead we assume that the \demoer can check if each action that the \learner takes in the trajectory is near-optimal: when the \learner takes $a$ at state $s$, an error is counted if and only if
$
Q^\star(s,a) < V^\star(s) - \epsilon.
$ 
This criterion can be viewed as a noisy version of the one used in previous sections, as taking expectation of $V^\star(s) - Q^\star(s,\pi(s))$ over the occupancy induced by $\pi$ will recover Equation~\ref{eq:mdploss}.

While this resolves the issue on the \learner's side, how should the \demoer provide his/her optimal trajectory? The most straightforward protocol is that the \demoer rolls out a trajectory from the initial distribution of the task, $\mu_t$. We argue that this is not a reasonable protocol for two reasons: (1) in expectation, the reward collected by the \demoer may be less than that by the \learner, because conditioning on the event that an error is spotted may introduce a selection bias; (2) the \demoer may not encounter the problematic state in his/her own trajectory, hence the information provided in the trajectory may be irrelevant.

To resolve this issue, we consider a different protocol where the \demoer rolls out a trajectory using an optimal policy from the very state where the \learner errs.  


Now we discuss how we can prove a parallel of Theorem~\ref{thm:ellipsoid} under this new protocol. First, let's assume that the demonstration were still given in the form a state occupancy vector 
starting at the problematic state. In this case, we can reduce to the setting of Section~\ref{sec:passive} by changing $\mu_t$ to a point mass on the problematic state.\footnote{At the first glance this might seem suspicious: the problematic state is random and depends on the learner's current policy, but in RL the initial distribution is usually fixed and the learner has no control over it. This concern is removed thanks to our adversarial setup on $(E_t, R_t)$ (of which $\mu_t$ is a component).} To apply the algorithm and the analysis in Section~\ref{sec:passive}, it remains to show that the notion of error in this section (a suboptimal action) implies the notion of error in Section~\ref{sec:passive} (a suboptimal policy): let $s$ be the problematic state and $\pi$ be the \learner's policy, we have
$
V^{\pi}(s) = Q^{\pi}(s, \pi(s)) \le Q^{\star}(s, \pi(s)) < V^\star(s) - \epsilon.
$ 
So whenever a suboptimal \emph{action} is spotted in state $s$, it indeed implies that the \learner's \emph{policy} is suboptimal for $s$ as the initial state. Hence, we can run Algorithm~\ref{alg:ellipsoid} as-is and Theorem~\ref{thm:ellipsoid} immediately applies.

To tackle the remaining issue that the demonstration is in terms of a single trajectory, we will not update $\Theta_t$ after each mistake as in Algorithm~\ref{alg:ellipsoid}, but only make an update after every mini-batch of mistakes, and aggregate them to form accurate update rules. See Algorithm~\ref{alg:trajectory}. The formal guarantee of the algorithm is stated in Theorem~\ref{thm:trajectory}, whose proof is deferred to Appendix~\ref{app:trajectory}.


\begin{theorem} \label{thm:trajectory}
	$\forall \delta \in (0,1)$, with probability at least $1-\delta$, the number of mistakes made by Algorithm~\ref{alg:trajectory} with parameters 
	$
	\Theta_0 = [-1, 1]^d, 
	$ 
	$
	H = \left\lceil \frac{\log(12/\epsilon)}{1-\gamma} \right\rceil,
	$ 
	and 
	$
	n = \left\lceil \frac{\log(\frac{4d(d+1)\log \frac{6\sqrt{d}}{\epsilon}}{\delta})}{32\epsilon^2} \right\rceil
	$
	where $d = |\Scal|$,\footnote{Here we use the simpler conversion explained right after Example~\ref{exm:mdp2bandit}. We can certainly improve the dimension to $d = |\Scal|-1$ by dropping the $\sref$ coordinate in all relevant vectors but that complicates presentation.} is at most $\tilde O(\frac{d^2}{\epsilon^2} \log(\frac{d}{\delta \epsilon}))$.\footnote{A $\log\log(1/\epsilon)$ term is suppressed in $\tilde O(\cdot)$.}
\end{theorem}

\section{Related work \& Conclusions}
Most existing work in IRL focused on inferring the reward function\footnote{While we do not discuss it here, in the economics literature, the problem of inferring an agent's utility from behavior-queries has long been studied under the heading of utility or preference elicitation \citep{chajewska2000making, von2007theory, regan2009regret, regan2011eliciting, rothkopf2011preference}. 
While our result in Section~\ref{sec:omni} uses similar techniques to elicit the reward function, we do so purely by observing the \demoer's behavior without external source of information (e.g., query responses).} using data acquired from a fixed environment \citep{ng2000algorithms, abbeel2004apprenticeship, coates2008learning, ziebart2008maximum, ramachandran2007bayesian, syed2007game, regan2010robust}. 
There is prior work on using data collected from multiple --- but exogenously fixed --- environments to predict agent behavior \citep{ratliff2006maximum}. There are also applications where methods for single-environment MDPs have been adapted to multiple environments \citep{ziebart2008maximum}. Nevertheless, all these works 
consider the objective of mimicking an optimal behavior in the presented environment(s), and do not aim at generalization to new tasks that is the main contribution of this paper. 
Recently, \citet{hadfield2016cooperative} proposed cooperative inverse reinforcement learning, where the human and the agent act in the same environment, allowing the human to actively resolve the agent's uncertainty on the reward function. However, they only consider a single environment (or task), and the  unidentifiability issue of IRL still exists. Combining their interesting framework with our resolution to unidentifiability (by multiple tasks) can be an interesting future direction.

\section*{Acknowledgement}
This work was supported in part by NSF grant IIS 1319365 (Singh \& Jiang) and in part by a Rackham Predoctoral Fellowship from the University of Michigan (Jiang).  Any opinions, findings, conclusions, or recommendations expressed here are those of the authors and do not necessarily reflect the views of the sponsors.

\bibliography{irl,inverse}
\bibliographystyle{plainnat}

\clearpage
\appendix

\section*{Appendix}
\section{Proof of Proposition~\ref{prop:equiv}}  \label{app:equivalence}
To show that $\theta - \theta' = c\cdot \bone_{|\Scal|}$ implies behavioral equivalence, we note that for any policy $\pi$ the occupancy vector $\occ^\pi_{\mu, P}$ always satisfies $\bone_{|\Scal|}^\top \occ^\pi_{\mu, P} = 1$, so 
$\forall \pi, |\theta^T \occ^\pi_{\mu, P} - \theta'^T  \occ^\pi_{\mu, P}| = c$, and therefore the set of optimal policies is the same.
	
To show the other direction, we prove that if $\theta - \theta' \notin \textrm{span}(\{\bone_{|\Scal|}\})$, then there exists $(E, R)$ such that the sets of optimal policies differ. In particular, we choose $R = -\theta'$, so that all policies are optimal under $R + \theta' = 0$. Since $\theta - \theta' \notin \textrm{span}(\{\bone_{|\Scal|}\})$, there exists states $i$ and $j$ such that $\theta(i) + R(i) \ne \theta(j) + R(j)$. Suppose $i$ is the one with smaller sum of rewards, then we can make $j$ an absorbing state, and have two deterministic actions in $i$ that transition to $i$ and $j$ respectively. 
Under $R+\theta$, the self-loop in state $i$ is suboptimal, and this completes the proof. \qed

\section{Proof of Lemma~\ref{lem:emu}} \label{app:emu}
The construction is as follows. Choose $\sref$ as the initial state, and make all other states absorbing. Let $R'(\sref) = 0$ and $R'$ restricted on $\Scal \setminus \{\sref\}$ coincide with $R$. The remaining work is to design the transition distribution of each action in $\sref$ so that the induced state occupancy matches exactly one column of $X$. 

Fixing any action $a$, and let $x$ be the feature that we want to associate $a$ with. The next-state distribution of $(\sref, a)$ is as follows: with probability $p = \frac{1-\|x\|_1}{1 - \gamma\|x\|_1}$ the next-state is $\sref$ itself, and the probability of transitioning to the $j$-th state in $\Scal \setminus \{\sref\}$ is $\frac{1-\gamma}{1-\gamma\|x\|_1} x(j)$. Given $\|x\|_1 \le 1$ and $x \ge 0$, it is easy to verify that this is a valid distribution.

Now we calculate the occupancy of policy $\pi(\sref) = a$. The normalized occupancy on $\sref$ is 
$$
(1-\gamma) (p + \gamma p^2 + \gamma^2 p^3 + \cdots) = \frac{p(1-\gamma)}{1-\gamma p} = 1 - \|x\|_1.
$$
The remaining occupancy, with a total $\ell_1$ mass of $\|x\|_1$, is split among $\Scal \setminus \{\sref\}$ proportional to $x$. Therefore, when we convert the MDP problem as in Example~\ref{exm:mdp2bandit}, the corresponding feature vector is exactly $x$, so we recover the original linear bandit problem. \qed


\section{Proof of Proposition~\ref{prop:ellip_identify}}
\label{app:ellip_identify}
Assume towards contradiction that $\|c_{T_0} - \trueR\|_\infty > \epsilon $. We will choose $(R_t, x_t\ii{1}, x_t\ii{2})$ to make the algorithm err. In particular, let $R_t = - c_{T_0}$, so that the algorithm acts greedily with respect to $\bzro_d$. Since $\bzro_d^\top x_t^a \equiv 0$, any action would be a valid choice for the algorithm.
	
On the other hand, $\|c_{T_0} - \trueR\|_\infty > \epsilon$ implies that there exists a coordinate $j$ such that
$
|\ee_j^\top (\trueR - c_{T_0})| > \epsilon,
$ 
where $\ee_j$ is a basis vector. 
Let $x_t\ii{1} = \bzro_d$ and $x_t\ii{2} = \ee_j$.  So the value of action $1$ is always $0$ under any reward function (including $\trueR + R_t$), and the value of action $2$ is
$
(\trueR + R_t)^\top x_t\ii{2}  = (\trueR - c_{T_0})^\top \ee_j,
$ 
whose absolute value is greater than $\epsilon$. At least one of the 2 actions is more than $\epsilon$ suboptimal, and the algorithm may take any of them, so the algorithm can err again. \qed

\section{Proof of Theorem~\ref{thm:spread}} \label{app:spread}
	It suffices to show that in any round $t$, if $\|c_t - \trueR\|_\infty > \frac{\epsilon \sqrt{(K-1)/2}}{\spread(X)}$, then $l_t > \epsilon$. The bound on $T$ follows directly from Theorem~\ref{thm:ellipsoid}. Similar to the proof of Proposition~\ref{prop:ellip_identify}, our choice of the task reward is $R_t = - c_t$, so that any $a \in A$ would be a valid choice of $a_t$, and we will choose the worst action. Note that $\forall a, a'\in \Abdt$,
	\begin{align*}
	l_t = (\trueR + R_t)^\top (x^{a_t^\star} - x^{a_t})
	\ge (\trueR - c_t)^\top (x^{a} - x^{a'}).
	\end{align*}
	So it suffices to show that there exists $a, a' \in \Abdt$, such that $(\trueR - c_t)^\top (x^{a} - x^{a'}) > \epsilon$. Let $y_t = \trueR - c_t$, and the precondition implies that $\|y_t\|_2 \ge \|y_t\|_\infty > \frac{\epsilon \sqrt{(K-1)/2}}{\spread(X)}$.
	
	Define a matrix $D$ of size $K \times (K(K-1))$, where each column 
	\begin{align}
	D = \begin{bmatrix}
	1 & 1 & \cdots & 0 \\
	-1 & 0 & \cdots & 0 \\
	0 & -1 & \cdots & 0 \\
	& & \ddots & \\
	0 & 0 & \cdots & -1 \\
	0 & 0 & \cdots & 1
	\end{bmatrix}.
	\end{align}
	contains exactly one $1$ and one $-1$ (the remaining entries are $0$), and the columns enumerate all possible positions of them. With the help of this matrix, we can rewrite the desired result ($\exists \, a, a' \in A$, s.t.~$(\trueR - c_t)^\top (x^{a} - x^{a'}) > \epsilon$) as 
	$
	\|y_t^\top X D\|_\infty \ge \epsilon.
	$ 
	We relax the LHS as $\|y_t^\top X D\|_\infty \ge \|y_t^\top X D\|_2 / \sqrt{K(K-1)}$, and will provide a lower bound on $\|y_t^\top X D\|_2$. Note that
	\begin{align*}
	y_t^\top X D = y_t^\top (\widetilde X + (X - \widetilde X))  D = y_t^\top \widetilde X D,
	\end{align*}
	because every row of $(X - \widetilde X)$ is some multiple of $\bone_K^\top$ (recall Definition~\ref{def:spread}), and every column of $D$ is orthogonal to $\bone_K$. Let $\widehat{(\cdot)}$ be the vector normalized to unit length,
	\begin{align*}
	\|y_t^\top \widetilde X D\|_2 
	= \|y_t\|_2 \|\hat y_t^\top \widetilde X D\|_2 
	= \|y_t\|_2 \|\hat y_t^\top \widetilde X\|_2 \|\widehat{\hat y_t^\top \widetilde X} \,  D\|_2.
	\end{align*}
	We lower bound each of the 3 terms. For the first term, we have the precondition $\|y_t\|_2 > \frac{\epsilon \sqrt{(K-1)/2}}{\spread(X)}$. The second term is $\widetilde{X}$ left multiplied by a unit vector, so its $\ell_2$ norm can be lower bounded by the smallest non-zero singular value of $\widetilde{X}$ (recall that $\widetilde{X}$ is full-rank), which is $\spread(X)$.  
	
	To lower bound the last term, note that 
	$
	D D^\top = 2K \mathbf{I}_K - 2 \bone_K \bone_K^\top,
	$  
	and rows of $\widetilde{X}$ are orthogonal to $\bone_K^\top$ and so is $y_t^\top \widetilde{X}$, so
	\begin{align*}
	\|\widehat{\hat y_t^\top \widetilde X} \,  D\|_2^2 
	\ge 
	\inf_{\|z\|_2 = 1,\, z \bot \bone_K} z^\top D D^\top z 
	= \inf_{\|z\|_2 = 1,\, z \bot \bone_K} z^\top (2K \mathbf{I}_K - 2 \bone_K \bone_K^\top) z 
	= 2K.
	\end{align*}
	Putting all the pieces together, we have
	\begin{align*}
	\|y_t^\top \widetilde X D\|_\infty
	\ge \frac{\|y_t\|_2 \|\hat y_t^\top \widetilde X\|_2 \|\widehat{\hat y_t^\top \widetilde X} \,  D\|_2}{\sqrt{K(K-1)}} 
	>  \frac{\epsilon \sqrt{(K-1)/2}}{\spread(X)} \cdot \spread(X) \cdot \frac{\sqrt{2K}}{\sqrt{K(K-1)}} = \epsilon. \qedhere
	\end{align*}

\section{Proof of Theorem~\ref{thm:lower_bound}}
\label{app:lower_bound}
As a standard trick, we randomize $\trueR$  by sampling each element i.i.d.~from $\textrm{Unif}([-1,1])$. We will prove that there exists a strategy of choosing $(X_t, R_t)$ such that any algorithm's expected number of mistakes is $\Omega(d \log(1/\epsilon)$, where the expectation is with respect to the randomness of $\trueR$ and the internal randomness of the algorithm. This immediately implies a worst-case result as max is no less than average (regarding the sampling of $\trueR$).

In our construction, $X_t = [\bzro_d,~ \ee_{j_t}]$, where $j_t$ is some index to be specified. Hence, every round the \learner is essentially asked to decided whether $\theta(j_t) \ge - R_t(j_t)$. The adversary's strategy goes in phases, and $R_t$ remains the same during each phase. Every phase has $d$ rounds where $j_t$ is enumerated over $\{1, \ldots, d\}$. To fully specify the nature's strategy, it remains to specify $R_t$ for each phase. 

In the 1st phase, $R_t\equiv 0$. For each coordinate $j$, the information revealed to the \learner is one of the following: $\trueR(j) > \epsilon$, $\trueR(j) \ge - \epsilon$, $\trueR(j) < - \epsilon$, $\trueR(j) \le \epsilon$. For clarity we first make an simplification, that the revealed information is either $\trueR(j) > 0$ or $\trueR(j) \le 0$; we will deal with the subtleties related to $\epsilon$ at the end of the proof.

In the 2nd phase, we fix $R_t$ as 
\begin{align*}
R_t(j) = \begin{cases}
-1/2 & \textrm{if $\trueR(j) \ge 0$,} \\
1/2 & \textrm{if $\trueR(j) < 0$.}
\end{cases}
\end{align*}
Since $\trueR$ is randomized i.i.d.~for each coordinate, the posterior of $\trueR + R_t$ conditioned on the revealed information is $\textrm{Unif}[-1/2, 1/2]$, for any algorithm and any interaction history. Therefore the 2nd phase is almost identical to the 1st phase except that the intervals have shrunk by a factor of $2$. Similarly in the 3rd phase we use $R_t$ to offset the posterior of $\trueR + R_t$ to $\textrm{Unif}([-1/4, 1/4])$, and so on. 

In phase $m$, the half-length of the interval is $2^{-m+1}$, and the probability that a mistake occurs is at least $1/2 - \epsilon/2^{-m+2}$ for any algorithm. The whole process continues as long as this probability is greater than $0$. By linearity of expectation, we can lower bound the total mistakes by the sum of expected mistakes in each phase, which gives
\begin{align} \label{eq:oblivious}
\sum_{2^{-m+1} \ge \epsilon} d (1/2 - \epsilon/2^{-m+2}) 
\ge \sum_{2^{-m+1} \ge 2 \epsilon} d \cdot 1/4
\ge \lfloor \log_2(1/\epsilon)\rfloor d/4.
\end{align}
The above analysis made a simplification that the posterior of $\trueR+R_t$ in phase $m$ is $[-2^{-m+1}, 2^{-m+1}]$. We now remove the simplification. Note, however, that if we  choose $R_t$ to center the posterior, $R_t$ reveals no additional information about $\trueR$, and in the worst case the interval shrinks to half of its previous size minus $\epsilon$. So the length of interval in phase $m$ is at least $2^{-m+2} (1+\epsilon) - 2\epsilon$, and the error probability is at least $1/2 - \epsilon / (2^{-m+1} (1+\epsilon) - \epsilon)$. The rest of the analysis is similar: we count the number of phases until the error probability drops below $1/4$, and in each of these phases we get at least $d/4$ mistakes in expectation. The number of such phases is given by 
$$ 1/2 - \epsilon / (2^{-m+1} (1+\epsilon) - \epsilon) \ge 1/4, $$
which is satisfied when $2^{-m+1} \ge 5\epsilon$, that is, when $m \le \lfloor \log_2 \frac{2}{5\epsilon} \rfloor$. This completes the proof. \qed

\section{Bounding the $\ell_\infty$ distance between $\trueR$ and the ellipsoid center} \label{app:bound_center}
To prove Theorem~\ref{thm:trajectory}, we need an upper bound on $\|\trueR - c\|_\infty$ for quantifying the error due to $H$-step truncation and sampling effects, where $c$ is the ellipsoid center. As far as we know there is no standard result on this issue. However, a simple workaround, described below, allows us to assume $\|\trueR - c \|_\infty \le 2$ without loss of generality.

Whenever $\|c\|_\infty > 1$, there exists coordinate $j$ such that $|c_j| > 1$. We can make a central cut $\ee_j^\top (\theta - c) < 0$ (or $>0$ depending on the sign of $c_j$), and replace the original ellipsoid with the MVEE of the remaining shape. This operation never excludes any point in $\Theta_0$, hence it allows the proofs of Theorem~\ref{thm:ellipsoid} and \ref{thm:trajectory} to work. 
We keep making such cuts and update the ellipsoid accordingly, until the new center satisfies $\|c\|_\infty \le 1 $. Since central cuts reduce volume substantially (Lemma~\ref{lem:vol}) and there is a lower bound on the volume, the process must stop after finite number of operations. After the process stops, we have $\|\trueR - c\|_\infty \le \|\trueR\|_\infty + \|c\|_\infty \le 2$.

\section{Proof of Theorem~\ref{thm:trajectory}}
\label{app:trajectory}
We first introduce a standard concentration inequality for martingales. 
\begin{lemma}[Azuma's inequality for martingales] \label{lem:azuma}
	Suppose $\{S_0, S_1, \ldots, S_n\}$ is a martingale and $|S_i - S_{i-1}| \le b$ almost surely. Then with probability at least $1-\delta$ we have 
	$
	|S_n - S_0| \le b \sqrt{2n\log(2/\delta)}.
	$
\end{lemma}

\begin{proof}
Since the update rule is still in the format of a central cut through the ellipsoid, Lemma~\ref{lem:vol} applies. It remains to show that the update rule preserves $\trueR$ and a certain volume around it, and then we can follow the same argument as for Theorem~\ref{thm:ellipsoid}.

Fixing a mini-batch, let $t_0$ be the round on which the last update occurs, and $\Theta = \Theta_{t_0}, c = c_{t_0}$. Note that $\Theta_t = \Theta$ during the collection of the current mini-batch and does not change, and $c_t = c$ similarly. 

For each $i=1, 2, \ldots, n$, define $z_i^{\star, H}$ as the expected value of $\hat z_i^{\star, H}$, where expectation is with respect to the randomness of the trajectory produced by the \demoer, and let $z_i^\star$ be the infinite-step expected state occupancy. Note that $\hat z_i^{\star, H}, z_i^{\star, H}, z_i^\star \in \RR^{|\Scal|}$. 

As before, we have $\trueR^\top (z_i^{\star} - z_i) > \epsilon$ and $c^\top (z_i^{\star} - z_i) \le 0$, so $(\trueR- c)^\top (z_i^{\star} - z_i) > \epsilon$. Taking average over $i$, we get $(\trueR- c)^\top (\frac{1}{n} \sum_{i=1}^n z_i^{\star} - \frac{1}{n} \sum_{i=1}^n z_i) > \epsilon$.

What we will show next is that $(\trueR- c)^\top (\frac{\bar Z^\star}{n} - \frac{\bar Z}{n}) > \epsilon / 3$ for $\bar Z^\star$ and $\bar Z$ on Line~\ref{lin:traj_update}, which implies that the update rule is valid and has enough slackness for lower bounding the volume of $\Theta_t$ as before. Note that
\begin{talign*} 
	&~ (\trueR- c)^\top (\frac{\bar Z^\star}{n} - \frac{\bar Z}{n}) 
	= (\trueR- c)^\top (\frac{1}{n}\sum_{i=1}^n z_i^\star - \frac{1}{n} \sum_{i=1}^n z_i) \\
	&~~~ - (\trueR- c)^\top (\frac{1}{n}\sum_{i=1}^n z_i^\star - \frac{1}{n}\sum_{i=1}^n z_i^{\star, H} ) \\
	&~~~ - (\trueR- c)^\top (\frac{1}{n}\sum_{i=1}^n z_i^{\star, H} - \frac{1}{n}\sum_{i=1}^n \hat z_i^{\star, H}).
\end{talign*}
Here we decompose the expression of interest into 3 terms. The 1st term is lower bounded by $\epsilon$ as shown above, and we will upper bound each of the remaining 2 terms by $\epsilon / 3$. For the 2nd term, since $\|z_i^{\star, H} - z_i^{\star}\|_1 \le \gamma^H$, the $\ell_1$ norm of the average follows the same inequality due to convexity, and we can bound the term using H\"older's inequality given $\|\trueR - c\|_\infty \le 2$ (see details of this result in Appendix~\ref{app:bound_center}). To verify that the choice of $H$ in the theorem statement is appropriate, we can upper bound the 2nd term as
\begin{talign*}
	2 \gamma^H = 2 ((1 - (1-\gamma))^{\frac{1}{1-\gamma}})^{\log(6/\epsilon)} \le 2 e^{-\log(6/\epsilon)} = \frac{\epsilon}{3}.
\end{talign*}

For the 3rd term, fixing $\trueR$ and $c$, the partial sum 
$
\sum_{j=1}^i (\trueR- c)^\top (z_i^{\star, H} - \hat z_i^{\star, H})
$
is a martingale. 
Since $\|z_i^{\star, H}\|_1 \le 1$, $\|\hat z_i^{\star, H}\|_1 \le 1$,  and $\|\trueR - c\|_\infty \le 2$, we can initiate Lemma~\ref{lem:azuma} by letting $b = 4$, and setting $n$ to sufficiently large to guarantee that the 3rd term is upper bounded by $\epsilon/3$ with high probability. 

Given $(\trueR- c)^\top (\frac{\bar Z^\star}{n} - \frac{\bar Z}{n}) > \epsilon / 3$, we can follow exactly the same analysis as for Theorem~\ref{thm:ellipsoid} to show that $B_{\infty}(\trueR, \epsilon/6)$ is never eliminated, and the number of updates can be bounded by $2d(d+1)\log \frac{12\sqrt{d}}{\epsilon}$. The number of total mistakes is the number of updates multiplied by $n$, the size of the mini-batches. Via Lemma~\ref{lem:azuma}, we can verify that the choice of $n$ in the theorem statement satisfies $|\sum_{j=1}^i (\trueR- c)^\top (z_i^{\star, H} - \hat z_i^{\star, H})| \le n\epsilon / 3$ with probability at least $1 - \delta / \left( 2d(d+1)\log \frac{12\sqrt{d}}{\epsilon} \right)$. Union bounding over all updates and the total failure probability can be bounded by $\delta$.
\end{proof}

\end{document}